\documentclass{article}

\usepackage{amsmath,amsfonts,bm}

\def\eqref#1{equation~\ref{#1}}
\def\1{\bm{1}}

\def\rve{{\mathbf{e}}}

\DeclareMathAlphabet{\mathsfit}{\encodingdefault}{\sfdefault}{m}{sl}
\SetMathAlphabet{\mathsfit}{bold}{\encodingdefault}{\sfdefault}{bx}{n}

\usepackage{microtype}
\usepackage{graphicx}
\usepackage{subfigure}
\usepackage{booktabs} 
\usepackage{hyperref}
\usepackage{url}
\usepackage{cleveref}
\usepackage{tabularx}
\usepackage{multirow}
\usepackage{caption}
\usepackage{amsmath}
\usepackage{bbm}
\usepackage{enumitem}
\usepackage[normalem]{ulem}
\usepackage{mathtools}
\useunder{\uline}{\ul}{}

\usepackage{amsthm}
\usepackage{thmtools}
\usepackage{thm-restate}

\declaretheorem[name=Claim]{claim}

\usepackage[accepted]{icml2021}
\icmltitlerunning{Rank4Class: A Ranking Formulation for Multiclass Classification}

\begin{document}

\twocolumn[
\icmltitle{Rank4Class: A Ranking Formulation for Multiclass Classification}

% It is OKAY to include author information, even for blind
% submissions: the style file will automatically remove it for you
% unless you've provided the [accepted] option to the icml2021
% package.

% List of affiliations: The first argument should be a (short)
% identifier you will use later to specify author affiliations
% Academic affiliations should list Department, University, City, Region, Country
% Industry affiliations should list Company, City, Region, Country

% You can specify symbols, otherwise they are numbered in order.
% Ideally, you should not use this facility. Affiliations will be numbered
% in order of appearance and this is the preferred way.
\icmlsetsymbol{equal}{*}

% Nan Wang, Zhen Qin, Le Yan, Honglei Zhuang, Xuanhui Wang, Michael Bendersky, Marc Najork
\begin{icmlauthorlist}
\icmlauthor{Nan Wang}{uva,intern}
\icmlauthor{Zhen Qin}{goo}
\icmlauthor{Le Yan}{goo}
\icmlauthor{Honglei Zhuang}{goo}
\icmlauthor{Xuanhui Wang}{goo}
\icmlauthor{Michael Bendersky}{goo}
\icmlauthor{Marc Najork}{goo}
\end{icmlauthorlist}

\icmlaffiliation{uva}{Department of Computer Science, University of Virginia, VA, USA}
\icmlaffiliation{goo}{Google Research, Mountain View, CA, USA}
\icmlaffiliation{intern}{Work was done when the author was a research intern at Google Research}

\icmlcorrespondingauthor{Nan Wang}{nw6a@virginia.edu}
\icmlcorrespondingauthor{Zhen Qin}{zhenqin@google.com}

% You may provide any keywords that you
% find helpful for describing your paper; these are used to populate
% the "keywords" metadata in the PDF but will not be shown in the document
\icmlkeywords{Machine Learning, ICML}

\vskip 0.3in
]

\printAffiliationsAndNotice{}  % leave blank if no need to mention equal contribution
% \printAffiliationsAndNotice{} % otherwise use the standard text.

\begin{abstract}
Multiclass classification (MCC) is a fundamental machine learning problem of classifying each instance into one of a predefined set of classes. In the deep learning era, extensive efforts have been spent on developing more powerful neural embedding models to better represent the instance for improving MCC performance. In this paper, we do not aim to propose new neural models for instance representation learning, but to show that it is promising to boost MCC performance with a novel formulation through the lens of ranking. In particular, by viewing MCC as to rank classes for an instance, we first argue that ranking metrics, such as Normalized Discounted Cumulative Gain, can be more informative than the commonly used Top-$K$ metrics. We further demonstrate that the dominant neural MCC recipe can be transformed to a neural ranking framework. Based on such generalization, we show that it is intuitive to leverage advanced techniques from the learning to rank literature to improve the MCC performance out of the box. Extensive empirical results on both text and image classification tasks with diverse datasets and backbone neural models show the value of our proposed framework.
\end{abstract}

\section{Introduction}
\label{sec:intro}
Multiclass classification (MCC) is the problem of classifying each instance into one of a predefined set of classes~\citep{hastie01statisticallearning}. It is one of the most fundamental machine learning problems that has broad applications in many fields such as natural language processing~\citep{sun2020finetune} and computer vision~\citep{he2016resnet}. For example, deciding the category of a news article or the subject of an image can be formulated as an MCC problem. 

Numerous MCC models have been proposed in the past, ranging from linear models to nonlinear decision trees and neural models~\citep{aly05survey}. In the modern deep learning era, while there are significant advances in neural architectures, dominating MCC methods share the same recipe: an input instance, being it a feature vector, an image, or a text sentence, is fed into a neural embedding model to produce its vector representation, which is followed by a classification layer to score against the candidate classes. The model is trained by using a loss function, typically the softmax cross entropy loss, between the labels and scores over all candidate classes~\citep{goodfellow2016deep}. During inference, the classes are sorted after an instance is scored against them. Metrics such as Top-K Accuracy (the percentage of test instances whose correct class label is in the top K predicted classes, also known as simply ``classification accuracy'' when K=1) are usually used for evaluating MCC performance. To better display the headroom of the models, Top-K Error ($1-$ Top-K Accuracy) is widely adopted to compare different classification models \cite{krizhevsky2012imagenet}. Following this recipe, most efforts in the MCC literature focus on designing more powerful neural network architectures for representation learning of the input instances~\citep{he2016resnet, densenet, dosovitskiy2020vit, minaee2021deep}. Few work has studied a formulation different from the recipe above.

Paralleling with developing more powerful embedding models, we seek for a novel formulation for MCC by examining it through the lens of ranking, or more specifically, learning to rank (LTR), a rich research field stemmed from information retrieval~\citep{liu2009learning}. As shown in the rest of the paper, firstly, such a formulation allows us to better evaluate model performances by borrowing more informative ranking metrics. Secondly, it improves MCC performance out of the box, agnostic to the embedding architecture used, by training models with advanced ranking losses and use more flexible matching schemes between the input instance and candidate classes.

In particular, we first formalize MCC as a ranking problem. As hinted above, dominating neural MCC methods score a given instance on a predefined set of classes and sort the classes during inference. This is equivalent to a LTR setting where a set of items (i.e., classes) are ranked given a query (i.e., the input instance). In fact, a Top-K Accuracy measure itself can be viewed as a ranking metric, similar to Precision@K used in information retrieval. However, other ranking metrics such as Normalized Discounted Cumulative Gain (NDCG) are more commonly used in LTR because they are more informative than Precision@K in real-world user-facing applications. They can be borrowed to enrich the MCC evaluation metrics but have not been widely adopted.

For modeling, we show a general ``equivalent view'' for MCC from ranking perspectives, where the dominant MCC recipe is equivalent to a neural ranking pipeline with \textit{a specific set of design choices}. This LTR view gives us insights into existing MCC models' groundings as well as limitations, which then inspires more design options. 
We further propose a general framework with several intuitive approaches under the LTR formulation to improve MCC performance. We mainly focus on two aspects: loss functions and model architectures. For loss functions, we leverage the rich literature in LTR of advanced ranking losses specifically designed for certain ranking metrics. For model architecture, we realize that the vast MCC literature focuses on only one component in the neural ranking architecture, i.e. the input instance embedding. With the LTR view, we can enhance the modeling capacity of other components as well. In this paper, we specifically study the effect of enhancing the interactions between instances and classes, a popular setup in LTR to match queries and items~\citep{li2014semantic}.

We report experimental results on a variety of MCC tasks, including different datasets and backbone models for different modalities (image and text). 
% Examples include BERT~\citep{bert} for text classification on GoEmotions~ \citep{demszky2020goemotions}, and ResNet~\citep{he2016resnet} for image classification on ImageNet~\citep{krizhevsky2012imagenet}. 
Results show that the proposed framework outperform or perform competitively with the baselines in all settings. We expect that the promising results can encourage the community to further examine MCC from LTR perspectives.

% The goal of this paper is not to propose a new representation learning architecture as most recent works do. Instead, 
% By making the connection between LTR and MCC, our contributions can be summarized as follows:
% \begin{itemize}%[noitemsep,topsep=0pt,parsep=0pt,partopsep=0pt]
%     \item We argue that ranking metrics can be more informative than the Top-K Accuracy/Error metrics and propose to use them in addition to existing metrics.
%     \item We show that dominant neural MCC recipe is equivalent to a specific neural ranking pipeline with several limitations.
%     \item Under the LTR formulation, we show it is natrual to adopt advanced ranking losses and richer instance-class interactions for improving MCC.
%     \item We conduct experiments on a wide range of text and image classification tasks to validate the effectiveness and generality of our LTR formulation for MCC.
% \end{itemize}

% The paper is organized as follows. We give an in-depth analysis of classification and ranking metrics in Section~\ref{sec:metrics}. In Section~\ref{sec:view}, we mathematically show an equivalent view of dominating MCC models from ranking perspectives and examine its design choices. In Section~\ref{sec:method}, we describe alternatives on the selection of loss functions and interaction architectures that can potentially improve MCC performance and evaluate them on various classification tasks in Section~\ref{sec:exp}. We discuss related work in Section~\ref{sec:relatedwork}. We conclude the paper and discuss future directions in Section~\ref{sec:conclusion}.

\section{Ranking Metrics for MCC}
\label{sec:metrics}
% We present popular metrics in classification and connect them with commonly used ranking metrics. 
\subsection{Metrics for Classification}
The basic binary classification problem classifies instances into positive and negative classes. MCC extends binary classification to more than two classes. Going from binary to multiclass is not trivial for evaluation. Binary classification metrics are usually \emph{class-oriented}. For example, metrics such as AUC and Accuracy are based on measures like true positives (TP) and false negatives (FN), which are computed with respect to positive and negative classes~\citep{sokolova2009metric}. These metrics are not directly used in MCC for more than two classes. In contrast, MCC metrics are usually \emph{instance-oriented}. The commonly used metrics are the Top-$K$ Accuracy/Error metrics, popularized by the ImageNet competition~\citep{ILSVRC15}. Some earlier works like \cite{crammer2002svm} defined the ``empirical error'' when working on multiclass SVM algorithms, which is equivalent to Top-1 Error. 
% For easier illustration of the relation between Top-$K$ metrics and ranking metrics, we will use Top-$K$ Accuracy in the rest of this section, which is simply $1 -$ Top-$K$ Error. 

\subsection{Relation to Ranking Metrics}
\begin{figure*}
    \centering
    \includegraphics[width=0.75\textwidth]{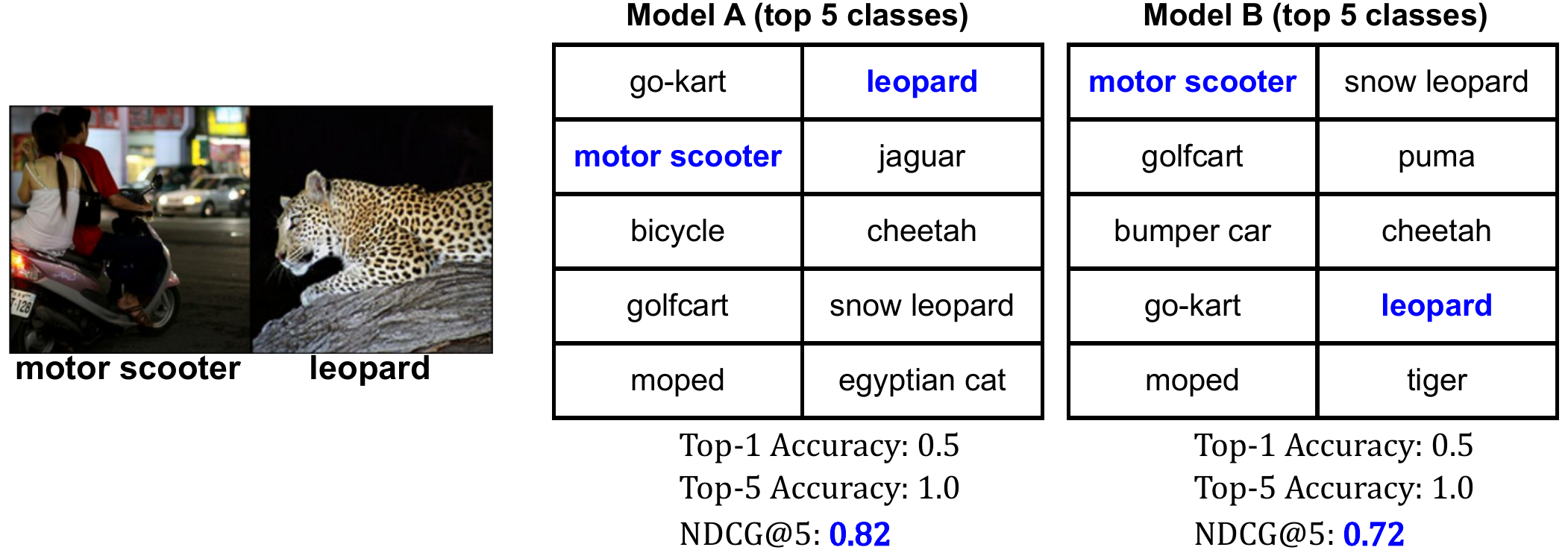}
    \vspace{-0.5em}
    \caption{We show an simulated example of evaluating image classification performance of two models with Top-1/5 Accuracy and NDCG@5. The Top-1/5 Accuracy of two models are exactly the same, and can not differentiate the two models. While NDCG@5 can successfully tell that model A has a better performance than model B.}
    \vspace{-0.7em}
    \label{fig:ndcg}
\end{figure*}
For easier illustration, we use Top-$K$ Accuracy in the rest of this section, which is simply $1 -$ Top-$K$ Error. As discussed above, MCC metrics are instance-oriented, and the metric value is simply averaged over all instances in a given evaluation set. Thus we study the metric calculated on a single instance to be more concise. 
% for a more concise illustration. 
In particular, given an instance and $n$ candidate classes, let $\mathbf y$ be the true labels of the classes and $\mathbf y[c] \in \{0, 1\}$ be the label for class $c$. We consider the standard MCC setting where there is exactly one class with label $1$ per instance \cite{aly05survey}. Let $\pi_A$ be a ranking of classes produced by model A, with the i-th most likely class being $\pi_A(i)$. Then the Top-$K$ Accuracy ($ 1 \leq K \leq n$, abbreviated as acc@$K$ in Eq~\ref{eq:topk-acc}) is
\begin{equation}
\label{eq:topk-acc}
    \textrm{acc@$K$}(\pi_A, \mathbf y) = \sum_{i=1}^{K} \mathbf y[\pi_A(i)].
\end{equation}
The Top-$K$ Accuracy can be thought of a type of a ranking metric. It is very close to the Precision@$K$: Top-$K$ Accuracy is the same as $\min(1, K \cdot \textrm{Precision@$K$})$ and Top-1 Accuracy is the same as Precision@1. Besides Precision@$K$, there are other commonly used ranking metrics such as Normalized Discounted Cumulative Gains (NDCG) and Mean Reciprocal Rank (MRR)~\citep{ndcg2002paper}. To the best of our knowledge, they are not commonly used for MCC evaluation, but can be more informative. 

We use NDCG@$K$ (simply called NDCG when $K=n$) as an example. Consider the instance as a query, and all the classes as candidate items with relevance labels $\mathbf y$, we have 
\begin{align} 
\label{eq:ndcg}
    \textrm{NDCG@$K$}(\pi_A, \mathbf y) = \frac{\textrm{DCG@$K$}(\pi_A, \mathbf y)}{\textrm{DCG@$K$}(\pi_*, \mathbf y)}, \\
    \nonumber \textrm{and}\; \textrm{DCG@$K$}(\pi, \mathbf y) = \sum_{i=1}^K \frac{2^{\mathbf y[\pi(i)]}-1}{\log_2(1+i)}.
\end{align}
where $\pi_*$ is the ideal ranking sorted by $\mathbf y$. It's easy to prove that NDCG@1 is equivalent to Top-1 Accuracy. 
\begin{restatable}[]{theorem}{ndcgvsacc}
\label{theorem:ndcgvsacc} 
Given an instance with class labels $\mathbf y \in \{0,1\}^n$ and $\sum_{c=1}^n \mathbf y[c]=1$, for any $1 < K \leq n$, NDCG@$K$ preserves more information than Top-$K$ Accuracy in evaluating the class rankings.
\end{restatable}
The proof of \Cref{theorem:ndcgvsacc} is in \Cref{appendix:proof}. The basic idea is to measure the information of a metric by the entropy of its value evaluated on any class ranking as a random variable. This information measure indicates that, given the metric's value, how much can we say about its input class ranking. 

\Cref{theorem:ndcgvsacc} shows that for $K>1$, NDCG@$K$ is strictly more informative than Top-$K$ Accuracy. To better illustrate this, let us consider two models A and B, which rank the correct class at position $p_A$ and $p_B$, respectively, and $p_A < p_B \leq K, 2 \leq K \leq n$. Thus, both models can rank the correct class in the top $K$ positions while model A ranks it higher than model B. With Top-$K$ Accuracy, we have 
% \begin{equation*}
$    \textrm{acc@$K$}(\pi_A, \mathbf y) = \textrm{acc@$K$}(\pi_B, \mathbf y) = 1$.
% \end{equation*}
But with NDCG@$K$, we have
% \begin{align*}
%     \mathit{NDCG@$K$}(\pi_A, \mathbf y) &= \frac{1}{\log_2(1 + p_A)}, \\\nonumber
%     \mathit{NDCG@$K$}(\pi_B, \mathbf y) &= \frac{1}{\log_2(1 + p_B)}, \\\nonumber
%     \Longrightarrow \mathit{NDCG@$K$}(\pi_A, \mathbf y) &> \mathit{NDCG@$K$}(\pi_B, \mathbf y).
% \end{align*}
$\textrm{NDCG@$K$}(\pi_A, \mathbf y) = 1/\log_2(1 + p_A)$, and $\textrm{NDCG@$K$}(\pi_B, \mathbf y) = 1/\log_2(1 + p_B)$, which gives us $\textrm{NDCG@$K$}(\pi_A, \mathbf y) > \textrm{NDCG@$K$}(\pi_B, \mathbf y)$.
Therefore, when model A consistently ranks the correct class higher than model B in top $K$ positions, Top-$K$ Accuracy may not be able to reflect the better performance of A, while NDCG@$K$ will always detect it. More importantly, although both metrics ignore the correct classes ranked below the $K$'th position, we can simply set $K = n$ to evaluate the whole ranked list of classes. In this case, NDCG is still a meaningful metric that reflects the position of the correct class in the full list, while Top-$K$ Accuracy becomes meaningless (always $1$).  
% In the simpler case of $p_a \leq K < p_b$, both metrics can detect the better performance of model $a$. However, we can always increase $K$ (up to $n$) such that the correct class is in the top K positions and NDCG@$K$ is always a reasonable metric. 

% To further illustrate the effectiveness of using NDCG@$K$ for MCC evaluation, 
In Figure~\ref{fig:ndcg}, we simulate an image MCC application in a \emph{user-facing interface} and compare the results from two models with different metrics. NDCG@5 is clearly more capable of distinguishing the two models, while both Top-1 Accuracy and Top-5 fails to. The reason is that there is a position discounting function in NDCG@$K$, while Top-$K$ metrics impose a simple hard cut at position $K$, which loses information about the exact ranked position of the correct class. Other ranking metrics such as MRR share similar characteristics with NDCG. In the rest of this paper, we choose NDCG@$K$ as our representatives for ranking metrics and use it in addition to Top-$K$ metrics in experiments.

\section{MCC from a Ranking Perspective}
\label{sec:view}

\subsection{Classical MCC Model Architectures}
\label{sec:class-models}
The general model architecture for MCC based on deep neural networks (DNN) is composed of three parts: an input instance to be classified, an encoder to extract latent representation (embedding) of the input, and a classification layer for generating scores on candidate classes, as shown in Figure~\ref{fig:equivalence} \textbf{Left}. We use $\mathbf{x}$ to represent the input instance such as a textual sentence or an image. The encoder can have different structure designs based on the modality of the input, such as transformers \citep{vaswani2017attention} to encode textual sentences; and convolutional neural networks \citep{krizhevsky2012imagenet} to encode images. It can be represented by a function $\mathcal{H}(\cdot)$ that maps $\mathbf{x}$ to a $d$-dimensional embedding vector $\mathbf{h}=\mathcal{H}(\mathbf{x})$. The classification layer is in most cases a dense layer with weight matrix $W\in\mathbb R^{n\times (d+1)}$. The classification scores are calculated by $\mathbf{s}=W\mathbf{h}'$, where $\mathbf{h}'\coloneqq [\mathbf{h},1]$ with an added bias dimension of value $1$ for the bias. $\mathbf{s}$ is an $n$-dimensional score vector for $n$ classes with $s_i=\rve_i^\top W\mathbf{h}'$ for the $i$-th class, where $\rve_i$ is a $n$-dimensional one-hot vector with the $i$-th dimension being $1$.

For training neural MCC models, the softmax cross entropy loss (SoftmaxCE) is used by default in almost all prior work \citep{goodfellow2016deep,zhang2021dive}. However, whether it is the most suitable loss for optimizing the evaluation metrics of interest is not carefully studied. Very recently, empirical results show that the mean squared error (MSE) can sometimes outperform SoftmaxCE in MCC tasks~\citep{hui2021evaluation}, but it is sensitive to an extra rescaling parameter for tasks with many classes. 

\subsection{An Equivalent View from Neural Ranking Models}
\label{sec:equivalence}

\begin{figure*}
    \centering
    \includegraphics[width=0.92\textwidth]{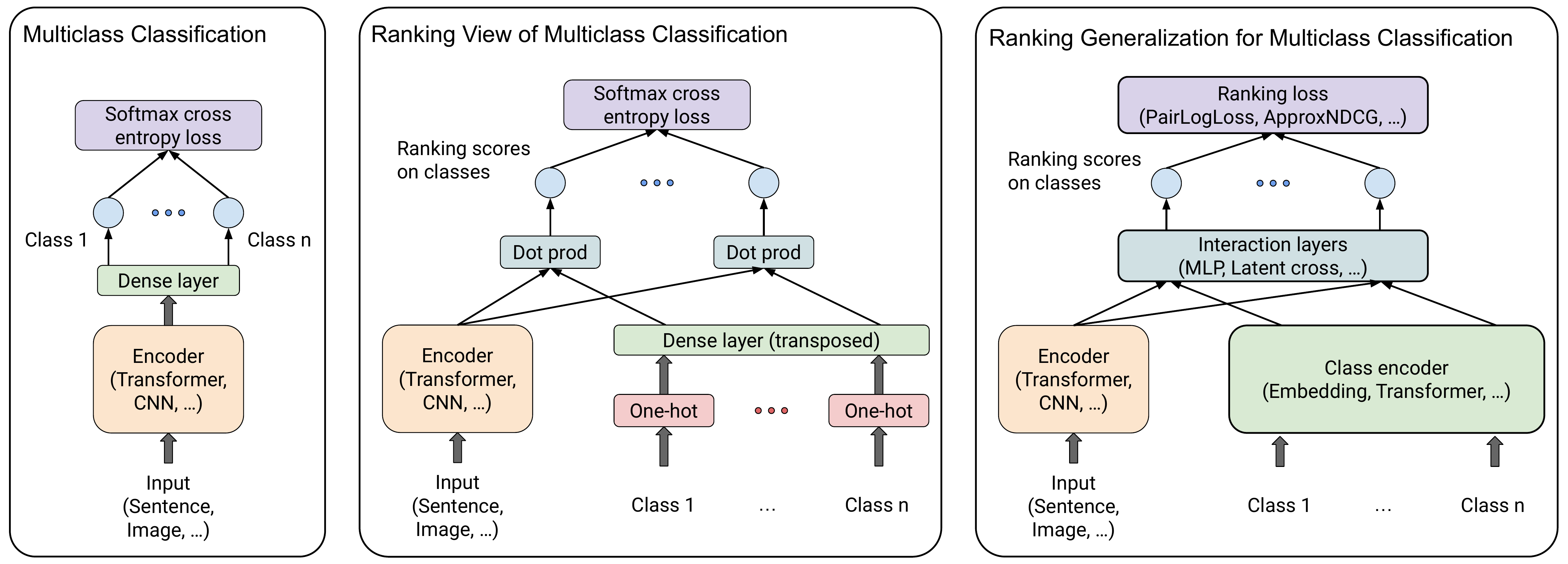}
    \vspace{-0.5em}
    \caption{
    % Classical MCC models from LTR perspectives. 
    \textbf{Left}: the classical DNN architecture for MCC; \textbf{Middle}: the equivalent neural ranking architecture; \textbf{Right}: The ranking generalization for MCC which allows the exploration of different class encoders, interaction layers, and loss functions.}
    \vspace{-0.5em}
    \label{fig:equivalence}
\end{figure*}
% link to figure: https://docs.google.com/presentation/d/1Kfny4kRKepUSPNMQl69zp7Oelx57OU3oInEK9dSkI6k/edit?usp=sharing

LTR learns a ranking model to rank a set of items (e.g., documents, news, etc.) based on their relevance to a query. Neural ranking models \citep{guo2019deep} adopt DNNs to match the query and items with their latent representations. 

%In general, a neural ranking model consists of a query encoder, an item encoder (which could be the same as the query encoder when the query and the item have the same format), and a scoring layer that produces relevance scores between a pair of query and item based on their representations from the encoders. The predicted relevance scores are fed into ranking losses for learning model parameters. 
\begin{claim}
\label{claim:equivalence} The classical DNN model for MCC as described in \Cref{sec:class-models} is equivalent to a specific neural ranking model trained with softmax cross entropy loss.
\end{claim}
\vspace{-0.5em}
To prove Claim \ref{claim:equivalence}, we show the equivalent view of a classical MCC model as a neural ranking model in Figure~\ref{fig:equivalence} \textbf{Middle}. In particular, we treat the input instance $\mathbf{x}$ as the query, and the $n$ candidate classes as input items to be ranked by the neural ranking model. The same as in the classical MCC model, we apply the encoder $\mathcal{H}(\cdot)$ to $\mathbf{x}$ to get a $d$-dimensional embedding $\mathbf{h}=\mathcal{H}(\mathbf{x})$, and $\mathbf{h}'=[\mathbf{h},1]$. For representing the classes, we use one-hot vector $\mathbf e_i$ for the $i$-th class and obtain its embedding $\mathbf c_i = W^\top \mathbf e_i$, where $W$ is the weight matrix from the classification layer of 
the classical MCC model. %The dense layer used as the class encoder can be mapped to the classification layer in the classical MCC model on the left in Figure~\ref{fig:equivalence}. Specifically, the dense layer used as class encoder have a weight matrix $W^\top$, which is a simple transpose of the weight matrix in the classification layer of the classical model. Then the embedding of class $i$ can be obtained by $\mathbf c_i = W^\top \mathbf e_i$. 
Finally, the ranking score of the $i$-th class can be calculated by simple \emph{dot product} $s'_i \equiv \mathbf c_i^\top \mathbf h' = \rve_i^\top W \mathbf h' =s_i$.
In this way, we show the classical MCC model can be transformed to an equivalent neural ranking model when both are trained by the SoftmaxCE. %We will show that neural ranking perspective to MCC enables much richer designs for different components. 

\section{Ranking Architectures for MCC}
\label{sec:method}
% Alternatives from Ranking Perspective
% From 3 perspectives: metric, loss, structure. \\
% - Propose a framework from the ranking perspective including ranking structures and ranking losses. \\
% - Possibilities of the new framework. How to add interactions, class-level auxiliary information, etc. \\
% - We mainly study simple interactions and ranking losses in this paper.
In this section, we discuss the possible new designs of different components in the equivalent ranking view of MCC models. We call the resulting new framework ranking for multiclass classification, or \emph{Rank4Class}. 

% \subsection{Directions to improve MCC} 
As discussed in Section~\ref{sec:intro}, existing work mostly focus on representation learning in the instance encoder. From the ranking perspective, however, we can see that there are several other promising directions to improve MCC performance. First, existing work mainly use SoftmaxCE, while there exist many advanced ranking losses that can be leveraged. Second, the interaction between the instance and class embeddings is simply a dot product, while richer interaction patterns can be explored. Third, while existing work focus on instance encoder, the class encoder is a naive linear projection that uses one-hot encoding of classes as input, where more powerful class encoders with richer inputs (e.g., class metadata) can be deployed. We illustrate the potential improvements in Figure~\ref{fig:equivalence} \textbf{Right}.
In this paper, we focus on ranking losses and interaction patterns for improving MCC performance, and leave the rest for future work.

\subsection{Ranking Losses for MCC}
\label{sec:ranking-loss}
Although SoftmaxCE is widely adopted for training MCC models, whether it is the best option to optimize existing evaluation metrics is not clear~\cite{hui2021evaluation}. On the other hand, we have shown that ranking metrics are better for evaluating MCC tasks. Thus, ranking losses are natural considerations since many of them are theoretically grounded to directly optimize certain ranking metrics. Next, we first discuss the soundness of using SoftmaxCE in MCC tasks with respect to ranking metrics. Then we present two example ranking losses, i.e., the pairwise logistic loss (PairLogLoss) and the approximate NDCG loss (ApproxNDCG). 

We re-denote the label of the $i$-th class for an instance by $y_i\in\{0,1\}, i\in\{1,\dots,n\}$, where only the correct class has label $1$. To compute the SoftmaxCE, the classification scores $s_i, i\in\{1,\dots,n\}$ produced by the model are first projected to the probability simplex $\mathbf{p}$ by softmax activation as $p_i = \frac{e^{s_i}}{\sum_{j=1}^n e^{s_j}}$. Then SoftmaxCE is defined as $$\ell_{ce}(\mathbf y, \mathbf p) = -\sum_{i=1}^n y_i \log p_i.$$

Intuitively, the SoftmaxCE is promoting the correct class against all other classes since only the term with $y_i=1$ is counted in it. It can be viewed as a listwise ranking loss, which aims to rank the correct class above all other classes. 
% In fact, \citet{bruch2019analysis} has shown that the SoftmaxCE is a bound on MRR and NDCG, which explains its promising performance in MCC tasks. 
\begin{restatable}[\citet{bruch2019analysis}]{theorem}{softmax}
\label{theorem:softmax}
Softmax cross entropy loss is a bound on mean reciprocal rank and mean normalized discounted cumulative gain in log-scale.
\end{restatable}
The proof of \Cref{theorem:softmax} can be found in \citet{bruch2019analysis} (Theorems 1-3). It explains the promising performance of SoftmaxCE in MCC tasks. To the best of our knowledge, there is no such bounds for MSE, so the findings in~\cite{hui2021evaluation} may need further theoretical investigation.

%As we have shown that ranking metrics, such as NDCG can be more informative and better for user-facing applications in evaluating classification tasks. Thus it is a natural thought to deploy losses developed to directly optimize ranking metrics for training the model. 
On the other hand, there is rich literature of developing ranking losses for optimizing ranking metrics \citep{burges2005learning, burges2006learning, burges2010ranknet}. Such ranking losses are usually directly derived from the target ranking metric, and bound the metric by approximation techniques \citep{qin2010general}. One of the most historical and popularly used ranking losses is the PairLogLoss. 
%It is an extension of the softmax cross entropy loss that maximizes the probabilities of correct pairwise orders:
\begin{equation}
    \ell_{pl}(\mathbf y, \mathbf s) = \sum_{i=1}^n \sum_{j=1}^n \mathbb I_{y_i > y_j}\log(1+e^{-\sigma(s_i - s_j)}), 
\end{equation}
where $\mathbb{I}$ is the indicator and $\sigma$ is a hyper-parameter. The PairLogLoss is proved to be able to minimize the rank of the relevant item \citep{wang2018lambdaloss}. 

Another popular ranking loss is the ApproxNDCG~\citep{qin2010general}, which directly optimizes the NDCG metric in Eq~\ref{eq:ndcg}.
% To complete NDCG, all that needed is the rank of items in the final ranked list $\pi_s$. 
The rank of an item $i$ can be computed as $\pi_s(i) = 1 + \sum_{j\neq i}\mathbb{I}_{s_i < s_j}$, where the indicator $\mathbb{I}_{s<t}$ is discrete but can be approximated by a sigmoid function to be smooth:
%and thus NDCG can not be directly optimized by popular gradient based algorithms. \citet{qin2010general} proposes to approximate the indicator function by 
\begin{equation}
    \mathbb I_{s<t} = \mathbb I_{t-s>0} \approx \frac{1}{1 + e^{-\alpha(t-s)}},
\end{equation}
where $\alpha>0$ is a parameter to control how tightly the indicator is approximated. 

To summarize, SoftmaxCE can be viewed as a ranking loss that bounds specific ranking metrics. Other ranking losses can also be valuable. We further investigate the empirical use of different ranking losses for MCC in \Cref{sec:exp}.

\subsection{Enhancing Instance-Class Interactions}
\label{sec:interaction}
After obtaining the embeddings $\mathbf{h}$ for the input instance and $\mathbf{c}_i$ for class $i$, it is also important to design the matching mechnism for producing the score $s_i$ between the embeddings of the instance and the class. For this purpose, we can add different interactions between $\mathbf{h}$ and $\mathbf{c}_i$ rather than simple dot-product for producing the score,
$
    s_i = \mathcal I(\mathbf{h}, \mathbf{c}_i)
$. $\mathcal I(\cdot, \cdot)$ is a function to represent the interaction between instance and class embeddings which produces a ranking score. In this paper, we consider the following two patterns as examples to enhance the interaction, but much richer interactions can be deployed based on the LTR formulation:
\begin{itemize}[noitemsep]
    \item LC+MLP: We first apply element-wise multiplication, also known as the latent cross (LC) operation~\citep{beutel2018latentcross} on $\mathbf{h}, \mathbf{c}_i$, and then follow with a multilayer perceptron (MLP) to get the score $s_i$. LC has shown to be a simple and efficient way to generate higher-order matching interactions in DNNs.
    \item Concat+MLP: We first concatenate the two embeddings $\mathbf{h}, \mathbf{c}_i$ and follow with an MLP to compute $s_i$.
\end{itemize}

% \subsection{Class encoder}
% If we decide not to include class encoder results, we remove this section.

% \begin{figure}[t!]
%     \centering
%         \centering
%         \includegraphics[width=0.48\textwidth]{iclr2022/figures/Rank4Class.pdf}
%         \caption{The ranking generalization for multiclass classification allows the exploration of different class encoders, interaction layers, and loss functions.}
%         \label{fig:interaction}
% \end{figure}
% link to figures: https://docs.google.com/presentation/d/1Kfny4kRKepUSPNMQl69zp7Oelx57OU3oInEK9dSkI6k/edit?usp=sharing
\begin{table}[t]
\centering
\caption{Statistics of datasets used in experiments.}
\label{tab:datasets}
\begin{tabular}{@{}ccccc@{}}
\toprule
Dataset    & \#classes & Train & Validation & Test  \\ \midrule
GoEmotions & 28        & 36.3K & 4.5K       & 4.6K  \\
MIND       & 18        & 78.8K & 25.7K      & 25.9K \\
% DBpedia    & 70        & 241K  & 36K        & 60.8K \\
ImageNet   & 1000      & 1.28M & 50K        & -     \\
CIFAR-10    & 10        & 50K   & -          & 10K   \\ \bottomrule
\end{tabular}
\vspace{-1em}
\end{table}

\begin{table*}[h!]
\centering
\caption{Results from classical MCC models and Rank4Class evaluated by Top-1/5 Error and NDCG@5. Bold font indicates the best value in each row. Relative Improvement means relative reduction in Top-1/5 Error and relative increase in NDCG@5.}
\label{tab:main-result}
\begin{tabular}{@{}cccccc@{}}
\toprule
Dataset & Encoder & Metrics & Classical MCC & Rank4Class & Relative Improvement \\ \midrule
\multirow{6}{*}{GoEmotions} & \multirow{3}{*}{BERT} & Top-1 Error & 41.00 & \textbf{40.65} & 0.85\% \\
 &  & Top-5 Error & 13.35 & \textbf{11.85} & 11.24\% \\
 &  & NDCG@5 & 73.90 & \textbf{74.62} & 0.97\% \\ \cmidrule(l){2-6} 
 & \multirow{3}{*}{ELECTRA} & Top-1 Error & 38.45 & \textbf{37.36} & 2.83\%\\
 &  & Top-5 Error & 10.45 & \textbf{8.65} & 17.22\% \\
 &  & NDCG@5 & 76.96 & \textbf{78.30} & 1.74\% \\ \midrule
\multirow{6}{*}{MIND} & \multirow{3}{*}{BERT} & Top-1 Error & 30.90 & \textbf{30.20} & 2.27\% \\
 &  & Top-5 Error & 6.15 & \textbf{5.05} & 17.89\%\\
 &  & NDCG@5 & 82.78 & \textbf{83.56} & 0.94\% \\ \cmidrule(l){2-6} 
 & \multirow{3}{*}{ELECTRA} & Top-1 Error & 27.09 & \textbf{26.40} & 2.55\% \\
 &  & Top-5 Error & 4.25 & \textbf{3.70} & 12.94\% \\
 &  & NDCG@5 & 85.73 & \textbf{86.13} & 0.47\% \\ \midrule
\multirow{3}{*}{ImageNet} & \multirow{3}{*}{ResNet50} & Top-1 Error & 23.74 & \textbf{23.58} & 0.67\% \\
 &  & Top-5 Error & 6.90 & \textbf{6.75} & 2.17\% \\
 &  & NDCG@5 & 85.74 & \textbf{85.85} & 0.13\%\\ \midrule
\multirow{3}{*}{CIFAR-10} & \multirow{3}{*}{VGG16} & Top-1 Error & 6.56 & \textbf{6.40} & 2.44\% \\
 &  & Top-5 Error & 0.20 & \textbf{0.15} & 25.00\% \\
 &  & NDCG@5 & 97.19 & \textbf{97.23} & 0.04\% \\
\bottomrule
\end{tabular}
\end{table*}

\section{Experiments}
\label{sec:exp}
% - Experiment design summary: datasets (CIFAR-10, ImageNet, DBpedia, GoEmotions, MIND) and base classification models (ResNet50, VGG16, ELECTRA, BERT). \\ 
% - Implementation of our ranking framework for multiclass classification. \\
% - Text classification. Results and analysis. What can we tell from different evaluation metrics, ranking losses and interaction structures.  \\
% - Image classification. Results and analysis.

We study Rank4Class on both text classification and image classification tasks in comparison to classical MCC models. We consider several datasets and instance encoders for evaluation. The datasets are summarized in Table \ref{tab:datasets}. For text classification, we include two recent large-scale datasets, GoEmotions \citep{demszky2020goemotions} and MIND \citep{wu2020mind}. The GoEmotions dataset contains instances with multiple labels, which are filtered out, since we focus on single-label MCC tasks. 
% We also include the DBpedia (Kaggle version) dataset \citep{lehmann2015dbpedia} for text classification in appendix due to space limit.
We adopt ELECTRA \citep{clark2020electra} and BERT \citep{bert} as text encoders in the experiments. For image classification, we use the popular ImageNet \citep{ILSVRC15} and CIFAR-10 \citep{krizhevsky2009cifar10}. We adopt ResNet50 \citep{he2016resnet} as image encoder for ImageNet and VGG16 \citep{simonyan2014very} for CIFAR-10. More details of the datasets and instance encoders are given in \Cref{appendix:implementation}. 
% dataset details. pre-trained encoders. how to train the rank4class models. 

For text datasets with both validation and test sets provided, we tune hyper-parameters on the validation set and report results on the test set. For ImageNet/CIFAR-10, we tune hyper-parameters and report the performance on the validation/test set respectively, which is the norm in the literature. 
More details on experimental settings such as data processing and hyper-parameter tuning are included in \Cref{appendix:protocol}. We use Top-1 Error (equivalent to $1-$ NDCG@1, lower the better), Top-5 Error, and NDCG@5 (higher the better) for evaluation. All result points are multiplied by $100$ for better illustration as commonly done in the literature.

% In the empirical study of the Rank4Class framework, we aim to answer the following research questions (RQs):
% \begin{itemize}
%     \item RQ1: Do ranking metrics such as NDCG@$K$ provide more information than Top-$K$ Error in evaluating models' performance on MCC tasks?
%     \item RQ2: How different (ranking) losses perform in optimizing MCC tasks with respect to different evaluation metrics?
%     \item RQ3: How do different interaction architectures of Rank4Class perform in MCC tasks? 
% \end{itemize}

% In \Cref{sec:overall-results}, we introduce different configurations of the Rank4Class framework with respect to loss functions and interaction patterns, and summarize its overall performance in MCC tasks. 
In \Cref{sec:overall-results}, we summarize the overall MCC performance of Rank4Class with respect to different configurations of loss functions and interaction patterns. In Section \ref{sec:exp-loss}, we study different ranking losses in optimizing MCC tasks with respect to different evaluation metrics. In Section \ref{sec:exp-structure}, we examine the effect of different interaction patterns between the instance and class embeddings. We discuss the effectiveness of ranking metrics in MCC evaluation in both sections. 
% We put the effect of combining different ranking losses and different interaction structures in \Cref{appendix:results}. 

\begin{table*}[t]
\centering
\caption{Results on text classification tasks trained with different losses.}
\label{tab:text-loss}
\begin{tabular}{@{}cccccc@{}}
\toprule
Dataset & Encoder & Metrics & SoftmaxCE & PairLogLoss & ApproxNDCG \\ \midrule
\multirow{6}{*}{GoEmotions} & \multirow{3}{*}{BERT} & Top-1 Error & \textbf{41.00} & 41.79 & 41.42 \\
 &  & Top-5 Error & 13.35 & \textbf{12.00} & 13.00 \\
 &  & NDCG@5 & 73.90 & \textbf{74.27} & 74.02 \\ \cmidrule(l){2-6} 
 & \multirow{3}{*}{ELECTRA} & Top-1 Error & 38.45 & 39.78 & \textbf{37.67} \\
 &  & Top-5 Error & 10.45 & \textbf{8.80} & 9.25 \\
 &  & NDCG@5 & 76.96 & 77.17 & \textbf{78.03} \\ \midrule
\multirow{6}{*}{MIND} & \multirow{3}{*}{BERT} & Top-1 Error & \textbf{30.90} & 31.21 & 30.97 \\
 &  & Top-5 Error & 6.15 & \textbf{5.25} & \textbf{5.25} \\
 &  & NDCG@5 & 82.78 & 83.25 & \textbf{83.36} \\ \cmidrule(l){2-6} 
 & \multirow{3}{*}{ELECTRA} & Top-1 Error & 27.09 & 27.72 & \textbf{26.65} \\
 &  & Top-5 Error & 4.25 & \textbf{3.75} & 4.30 \\
 &  & NDCG@5 & 85.73 & 85.74 & \textbf{85.89} \\ 
% \midrule
% \multirow{6}{*}{DBpedia} & \multirow{3}{*}{BERT} & Top-1 Error & 0.9791 & \textbf{0.9798} & 0.9788 \\
%  &  & Top-5 Error & 0.9990 & \textbf{0.9995} & \textbf{0.9995} \\
%  &  & NDCG@5 & 0.9911 & \textbf{0.9916} & 0.9914 \\ \cmidrule(l){2-6} 
%  & \multirow{3}{*}{ELECTRA} & Top-1 Error & 0.9775 & \textbf{0.9780} & 0.9778 \\
%  &  & Top-5 Error & 0.9985 & \textbf{0.9990} & 0.9985 \\
%  &  & NDCG@5 & 0.9901 & \textbf{0.9905} & \textbf{0.9905} \\
\bottomrule
\end{tabular}
\end{table*}

\begin{table*}[t]
\centering
\caption{Results on image classification tasks trained with different losses.}
\label{tab:image-loss}
\begin{tabular}{cccccc}
\toprule
Dataset & Encoder & Metrics & SoftmaxCE & PairLogLoss & ApproxNDCG \\ \midrule
\multirow{3}{*}{ImageNet} & \multirow{3}{*}{ResNet50} & Top-1 Error & 23.74 & 23.64 & \textbf{23.62} \\
 &  & Top-5 Error & 6.90 & \textbf{6.80} & 6.85 \\
 &  & NDCG@5 & 85.74 & \textbf{85.82} & 85.80 \\ \midrule
\multirow{3}{*}{CIFAR-10} & \multirow{3}{*}{VGG16} & Top-1 Error & 6.56 & 6.43 & \textbf{6.41} \\
 &  & Top-5 Error & 0.20 & \textbf{0.15} & \textbf{0.15} \\
 &  & NDCG@5 & 97.19 & 97.21 & \textbf{97.23} \\ \bottomrule
\end{tabular}
\end{table*}

\subsection{Overall Performance of Rank4Class}
\label{sec:overall-results}
Besides SoftmaxCE, PairLogLoss, and ApproxNDCG discussed in \Cref{sec:ranking-loss}, we also include Gumbel-ApproxNDCG~\citep{bruch2020stochastic} loss and MSE~\citep{hui2021evaluation} in experiments. We study different combinations between these five losses and the two interaction patterns introduced in \Cref{sec:interaction} in the Rank4Class framework. In \Cref{tab:main-result}, we report the best performance of different configurations for Rank4Class under each metric in comparison to the baseline MCC models. The complete results of all combinations are included in \Cref{appendix:results}, where we use $*$ to mark the combinations that achieve the best performance in each task under each metric. As shown in the table, Rank4Class can improve the MCC performance in virtually all tasks through specific combinations of losses and interaction patterns. This shows the increased capacity from Rank4Class in MCC tasks evaluated on different metrics, which is achieved by adding more flexible design options in different components from LTR perspectives.

\subsection{Effect of Ranking Losses}
\label{sec:exp-loss}
In this section, we study the use of ranking losses for optimizing MCC performance. In particular, we use PairLogLoss and ApproxNDCG as two most representative ranking losses in comparison to the SoftmaxCE. The results on other losses can be find in \Cref{appendix:results}. We only vary the loss function in the base Rank4Class structure in Figure \ref{fig:equivalence} (Middle), so the ``SoftmaxCE'' method is the baseline that is equivalent to classical MCC models.  

The results are shown in Table \ref{tab:text-loss} and \ref{tab:image-loss} for text and image classification tasks respectively. Overall, PairLogLoss or ApproxNDCG can outperform SoftmaxCE in nearly all tasks and metrics except for the Top-1 Error on GoEmotions and MIND datasets with BERT model. Besides, PairLogLoss is generally good at reducing Top-5 Error than SoftmaxCE and ApproxNDCG, achieving the lowest Top-5 Error in all tasks. ApproxNDCG performs well on both Top-1 Error and NDCG@5 (achieves the top in four out of six tasks), which shows its effectiveness in directly optimizing NDCG metrics. Moreover, we see that the improvements from PairLogLoss and ApproxNDCG are more significant on text classification tasks than that on image classification tasks. Our observation is similar to that in~\citep{hui2021evaluation}, which hypothesized the reason being that structures and parameters of popular image encoders are all heavily tuned with the SoftmaxCE.
%classification performance depends more on the quality of image representations, which is mainly decided by the structure of the image encoders. As we are using the same image encoder and study the other components of the model, the MCC performance can be difficult to improve dramatically.  
% This is because image classification tasks are more standard and the performance is difficult to improve based on the same image representation. 

Finally, we observe that different metrics are not always consistent in evaluating MCC tasks. For example, on GoEmotions with ELECTRA, PairLogLoss performs better than ApproxNDCG on Top-5 Error, while ApproxNDCG outperforms PairLogLoss on NDCG@5. This means that the PairLogLoss tends to rank the correct class in top 5 positions in more instances than ApproxNDCG, but ApproxNDCG can put the correct class relatively higher in the rankings. Furthermore, on MIND with BERT, Top-5 Error can not tell if PairLogLoss or ApproxNDCG is better. But NDCG@5 can successfully differentiate them since it also takes the absolute rank of the correct class in top 5 positions into account and thus is more informative. 

% \begin{table}[t]
% \centering
% \caption{Confusion matrix for different losses: For each metric, the number of times that the corresponding loss achieves the highest metric value.}
% \label{tab:text-loss}
% \begin{tabular}{@{}cccc@{}}
% \toprule
% & cross entropy & pairwise logistic & ApproxNDCG \\ \midrule
% Top-1 Error & 2 & 0 & 4\\
% Top-5 Error & 0 & 6 & 2\\
% NDCG@5 & 0 & 2 & 4\\ \midrule
% Sum & 2 & 8 & 10 \\
% \bottomrule
% \end{tabular}
% \end{table}

\begin{table*}[t]
\centering
\caption{Results on text classification tasks of different interactions trained with ApproxNDCG.}
\label{tab:text-structure}
\begin{tabular}{@{}cccccc@{}}
\toprule
Dataset & Encoder & Metrics & \begin{tabular}[c]{@{}c@{}}dot product\end{tabular} & LC+MLP & Concat+MLP \\ \midrule
\multirow{6}{*}{GoEmotions} & \multirow{3}{*}{BERT} & Top-1 Error & 41.42 & \textbf{40.65} & 40.92 \\
 &  & Top-5 Error & 13.00 & 12.40 & \textbf{12.20} \\
 &  & NDCG@5 & 74.02 & \textbf{74.62} & 74.61 \\ \cmidrule(l){2-6} 
 & \multirow{3}{*}{ELECTRA} & Top-1 Error & \textbf{37.67} & \textbf{37.67} & 37.80 \\
 &  & Top-5 Error & 9.25 & \textbf{8.65} & 9.15 \\
 &  & NDCG@5 & 78.03 & \textbf{78.30} & 77.97 \\ \midrule
\multirow{6}{*}{MIND} & \multirow{3}{*}{BERT} & Top-1 Error & 30.97 & 30.81 & \textbf{30.77} \\
 &  & Top-5 Error & 5.25 & 5.25 & \textbf{5.05} \\
 &  & NDCG@5 & 83.36 & \textbf{83.43} & 83.35 \\ \cmidrule(l){2-6} 
 & \multirow{3}{*}{ELECTRA} & Top-1 Error & 26.65 & 26.74 & \textbf{26.40} \\
 &  & Top-5 Error & 4.30 & \textbf{3.70} & 4.30 \\
 &  & NDCG@5 & 85.89 & \textbf{86.13} & 86.05 \\ 
% \midrule
% \multirow{6}{*}{DBpedia} & \multirow{3}{*}{BERT} & Error@1 & 0.9788 & \textbf{0.9794} & 0.9790 \\
%  &  & Error@5 & \textbf{0.9995} & 0.9990 & 0.9990 \\
%  &  & NDCG@5 & 0.9914 & \textbf{0.9915} & 0.9913 \\ \cmidrule(l){2-6} 
%  & \multirow{3}{*}{ELECTRA} & Error@1 & 0.9778 & \textbf{0.9780} & 0.9776 \\
%  &  & Error@5 & 0.9985 & \textbf{0.9990} & 0.9985 \\
%  &  & NDCG@5 & 0.9905 & \textbf{0.9906} & 0.9904 \\ 
\bottomrule
\end{tabular}
\end{table*}

\begin{table*}[t]
\centering
\caption{Results on image classification tasks of different interactions trained with SoftmaxCE.}
\label{tab:image-interaction}
\begin{tabular}{cccccc}
\toprule
Dataset & Encoder & Metrics & dot product & LC+MLP & Concat+MLP \\ \midrule
\multirow{3}{*}{ImageNet} & \multirow{3}{*}{ResNet50} & Top-1 Error & 23.74 & \textbf{23.62} & 23.66 \\
 &  & Top-5 Error & 6.90 & 6.80 & \textbf{6.75} \\
 &  & NDCG@5 & 85.74 & \textbf{85.84} & 85.83 \\ \midrule
\multirow{3}{*}{CIFAR-10} & \multirow{3}{*}{VGG16} & Top-1 Error & 6.56 & \textbf{6.40} & 6.43 \\
 &  & Top-5 Error & \textbf{0.20} & \textbf{0.20} & \textbf{0.20} \\
 &  & NDCG@5 & 97.19 & 97.22 & \textbf{97.23} \\ \bottomrule
\end{tabular}
\end{table*}

\subsection{Effect of Instance and Class  Interactions}
\label{sec:exp-structure}
In this section, we study the effect of different interaction patterns between instance and class embeddings for producing ranking scores. Specifically, we use the two interaction patterns in \cref{sec:interaction} with simple 2-layer MLPs. 

% the first interaction we study is using latent-cross embedding \citep{beutel2018latentcross} between a pair of instance and class embeddings, which is shown to be efficient in capturing feature crosses between query and items in information retrieval tasks. The latent cross embedding then goes through a 2-layer MLP for producing a scalar as the ranking score of this class. The second interaction we apply is to concatenate the instance and class embeddings and then feed into a 2-layer MLP for producing the ranking score. 
We compare these two types of interactions with the dot-product baseline between the instance and class embeddings by fixing the loss function. In particular, we use ApproxNDCG in text classification tasks and SoftmaxCE in image classification tasks. The study of different interactions on other losses are included in \Cref{appendix:results}. The results are shown in Table \ref{tab:text-structure} and \ref{tab:image-interaction} for text and image classification tasks. The results demonstrate that the two added interactions can outperform or achieve competitive performance than simple dot-product in all tasks and metrics. This shows the effectiveness of adding enhanced interactions based on the Rank4Class framework. In particular, latent-cross embedding tends to perform better than the concatenation of instance and class embeddings, achieving top performance in four out of six on Top-1 Error and five out of six tasks on NDCG@5. Again, we see that different evaluation metrics are not always consistent with each other, and NDCG@5 can be more informative than Top-5 Error, as observed and discussed in Section \ref{sec:exp-loss}.

\section{Related Work}
\label{sec:relatedwork}
Modern deep neural networks for MCC converge to the same recipe: given an input instance, a neural encoder is used to output its vector representation for generating scores on a set of classes. The vast research literature focus on developing more effective encoders in diverse domains, such as computer vision~\citep{he2016resnet, krizhevsky2012imagenet, densenet, dosovitskiy2020vit, 50333}, natural language processing~\citep{sun2020finetune, tay2021pre, minaee2021deep}, and automatic speech recognition~\citep{moritz2019triggered}, among others. The softmax cross entropy loss is the dominant loss function discussed in these papers. Only very recently, \citet{hui2021evaluation} study the mean squared error as another loss for MCC. Our work is orthogonal to the extensive research on neural encoders in that we provide a new formulation from the LTR perspective. Such a perspective inspires more diverse loss functions and model architectures that can model interactions between inputs and classes more effectively. 

Learning to rank (LTR) is a long-established interdisciplinary research area at the intersection of machine learning and information retrieval~\citep{liu2009learning}. Neural rankers are dominating in ranking virtually all modalities recently, including text ranking~\citep{jimmylinbook}, image retrieval~\citep{gordo2016deep}, and tabular data ranking~\citep{dasalc}. Many LTR papers focus on more effective loss functions~\citep{qin2010general, bruch2020stochastic} to rank items with respect to a query. One of the focuses of this paper is to introduce new techniques stemmed from LTR to solve MCC problems. 

%Some discussion on metrics? The relation between AUC and the error rate was discussed in \citep{cortes2003auc}.

% As we mentioned above, LTR is not restricted to the field of information retrieval. For example, the computer vision community also has a rich literature on image search~\citep{gordo2016deep,reidsurvey}. However, all these work rank actual items (e.g., an image or a document) rather than rank classes. 
%Extreme classification is a sub-field of classification where the number of classes is extremely large~\citep{ec2013}. It also has a ranking view to select a very small subset of classes from the large pool. The problem that extreme classification deals with is very similar to a recommendation problem. In another sense, it can be thought as an alternative approach for recommendation problems, while our work is on the traditional MCC problems. 

LTR has been explored in multi-lable classification (MLC)~\citep{agrawal2013multilabel,zhang2018deep,jain2019slice}, where each instance can belong to multiple classes. 
MLC is generally treated as a different problem from classical MCC: the number of labels assigned to an instance could be arbitrary and one research focus is to decide the threshold to cutoff the prediction list.
For example,  \citep{azarbonyad2021learning} focuses on feature engineering in learning a linear scoring function for ranking classes, and studies how the top-k threshold affect the MLC performance. 
Instead, we considers the recent neural ranking models for MCC, and do not consider multi-label scenario. An important related direction in MLC is extreme classification (EC)~\citep{zhang2018deep}, which formulates document retrieval or recommendation as MLC problems, where each document/item is treated as a candidate class. EC considers a huge number of classes and the main focus is on the scalability and efficiency~\citep{weston2011wsabie,weston2013label}. LTR techniques have been proposed for MLC in the past~\citep{yang2012multilabel}. In contrast, standard MCC has a smaller number of classes and each instance has a single correct class label. LTR have not been well studied for MCC, and we are the first to propose a unified ranking formulation for it. 
% Its recently popular sub-field, extreme multi-label classification~\citep{zhang2018deep}, also has a different and specific setting (e.g., the number of labels is huge) and research focus (e.g., efficiency to rank in the huge label space). LTR techniques have been proposed for MLC in the past~\citep{yang2012multilabel}. While the perspective is similar to ours, the work is not in the deep learning setting and thus does not have the equivalence view of MCC in LTR. Our work makes the connection and also proposes extensions for MCC in the deep learning setting.

\section{Conclusion}
\label{sec:conclusion}
In this paper we examine the classical MCC problem through the lens of ranking. Such a perspective brings benefits to MCC from three aspects: ranking metrics, ranking losses, and ranking architectures. We first show that ranking metrics can be more informative for MCC evaluations. Then, in the deep learning setting, we show an equivalent view of MCC in LTR setting. Such a connection provides new perspectives for MCC with respect to loss functions and model architectures. We studies these new formulations of MCC on various datasets and  observe promising results.

Our work opens up several research directions. First, the new ranking architectures allow to take more class information such as class metadata into account and it is interesting to study how this additional information can improve MCC. Second, it is also possible to apply the proposed framework Rank4Class to binary classification. Third, classes are usually not independent and our framework can incorporate the relationship between classes into the MCC through attention mechanisms, which is worth studying.

\bibliography{reference}
\bibliographystyle{icml2021}

% DELETE THIS PART. DO NOT PLACE CONTENT AFTER THE REFERENCES!
\newpage
\appendix
\onecolumn
\section{Proof of \Cref{theorem:ndcgvsacc}}
\label{appendix:proof}
\ndcgvsacc*
\begin{proof}
Consider NDCG@$K$ and Top-$K$ Accuracy as functions of any input class ranking $\pi\in\Pi$, where $\Pi$ is the space of all possible rankings (of size n!). Let $P$ be a distribution over $\Pi$ from which the evaluation set is sampled. Let the probability of the correct class being at position $i$ in the evaluation set be $p_i, p_i \geq 0, i\in[n]$, and $\sum_{i=1}^n p_i = 1$. Thus $\mathbb E[p_i] = \sum_{\mathbf y[\pi(i)]=1} P(\pi)$, where $P(\pi)$ is the probability that a class ranking $\pi$ being sampled from $\Pi$. 

Note that NDCG@$K$ has $K$ distinct values corresponding to the correct class being ranked at the top $K$ positions plus $0$ for not in top $K$ positions. Therefore, the entropy of the values of NDCG@$K$ can be obtained as:
\begin{equation}
    H(\textrm{NDCG@$K$}) = -\sum_{i=1}^K p_i \log p_i - p_{-} \log p_{-}, 
\end{equation}
where $p_- = \sum_{i=K+1}^n$ is the probability that the correct class is not ranked in top $K$ positions. $p_{-} \log p_{-}=0$ when $p_{-}=0$ as a convention in information theory. 

However, Top-$K$ Accuracy only has two possible values: $1$ when the correct class is in top $K$ positions, and $0$ otherwise. Hence the entropy of Top-$K$ Accuracy values is
\begin{equation}
    H(\textrm{acc@$K$}) = - p_{+} \log p_{+} - p_{-} \log p_{-},
\end{equation}
where $p_+ = \sum_{i=1}^K$ is the probability that the correct class is ranked in top $K$ positions. 

Finally, we have 
\begin{align}
    &H(\textrm{NDCG@$K$}) - H(\textrm{acc@$K$}) \\\nonumber
    =& -\sum_{i=1}^K p_i \log p_i + p_{+} \log p_{+} \\\nonumber
    =& \sum_{i=1}^K p_i \big(\log p_+ - \log p_i\big) \geq 0.
\end{align}
Therefore, NDCG@$K$ has more information than Top-$K$ Accuracy in terms of the entropy of the evaluation results.
\end{proof}

\section{Datasets and Encoders}
\label{appendix:implementation}
\subsection{Datasets}
\begin{itemize}
    \item \textbf{GoEmotions} \citep{demszky2020goemotions} is the largest manually annotated dataset for fine-grained emotion classification. It contains 58k English Reddit comments, labeled with 27 emotion categories and a default ``Neutral" category. On top of the raw data, the authors also provided a version that only includes comments with two or more raters agreeing on at least one label, split into train/test/validation sets. In our experiments, we filter out comments with more than 1 emotion labels in all three sets. 
    \item \textbf{MIND} \citep{wu2020mind} is a large-scale dataset for news recommendation. MIND contains about 130k English news articles. Every news article contains rich textual content including title, abstract, body, category and entities. We use the concatenation of title and abstract as the content of the news and use the category of the news as the classification label. We split all news instances into train/test/validation sets by roughly 60/20/20 for experiments.
    % \item \textbf{DBpedia} \citep{lehmann2015dbpedia} is a project aiming to extract structured content from the information created in Wikipedia. We use the Kaggle version\footnote{https://www.kaggle.com/danofer/dbpedia-classes/} for experiments, which is an extract of the original data after cleaning and is much tougher. It provides taxonomic, hierarchical categories ("classes") for 342,782 articles. There are 3 levels, with 9, 70 and 219 classes respectively. We use the second level categories with 70 classes and keep the train/test/validation split (with 241k/60.8k/36k instances) on Kaggle in our experiments. 
    \item \textbf{ImageNet} \citep{krizhevsky2012imagenet} is an image dataset with around 1.28 million images in the training set and 50k images in the validation set. Each image is labeled by one of 1,000 classes. The images are cropped and resized to $224\times 224 \times 3$ pixels as the input following the preprocessing in \url{https://github.com/tensorflow/models/tree/master/official/vision/image_classification/resnet}. We follow the common practices of using ImageNet and report the results on the validation set.
    \item \textbf{CIFAR-10} \citep{krizhevsky2009cifar10} contains 50k training images and 10k test images. There are 10 classes and each class has 6k images. All images have the same size of $32\times 32\times 3$. 
\end{itemize}

\subsection{Instance Encoders}
\begin{itemize}
    \item \textbf{BERT} \citep{vaswani2017attention} is a model based on transformers pretrained on a large corpus of English data. We use the BERT-Base-uncased model from \url{https://github.com/google-research/bert} for finetuning in our experiments. The maximum sequence length is set to 32 for GoEmotions and 128 for MIND.
    \item \textbf{ELECTRA} \citep{clark2020electra} is another pretrained transformer model. We use the ELECTRA-Base-uncased model from \url{https://github.com/google-research/electra} in our experiments. The maximum sequence length is set in the same way as BERT. 
    \item \textbf{ResNet50} \citep{he2016resnet} is a 50 layers deep convolutional neural network with residual connections. We use the implementation of ResNet50 from the tensorflow official models at \url{https://github.com/tensorflow/models/tree/master/official/vision/image_classification/resnet}. We adopt the pretrained network on ImageNet and finetune it in our Rank4Class framework. 
    \item \textbf{VGG16} \citep{simonyan2014very} is a convolutional neural network with 16 layers. We use the implementation of VGG16 from \url{https://github.com/geifmany/cifar-vgg} with input size $32\times 32$ in our experiments on CIFAR-10. 
\end{itemize}

\section{Experimental Settings}
\label{appendix:protocol}
% - hyper-parameter tuning, lr, lr schedule
% - model selection
% - training facilities, optimizers
% - how the losses are implemented and datasets processed.
We provide the implementation and hyper-parameter tuning details for reproducibility. For text classification tasks, we tokenize the raw sentences into word ids based on BERT vocabulary, and create the input masks and segment ids following standard BERT input formats. For MIND, we concatenate the title and abstract of each news by adding a ``[SEP]'' token between them. ELECTRA uses exactly the same input formats as BERT. The data processing of image datasets follow the standard methods as given in the original papers. We implement the Rank4Class pipeline based on TF-ranking \citep{rama2019tfrank}. We adopt TF-ranking's implementation of SoftmaxCE, MSE and 
all the ranking losses. 

For instance encoders, we use the default hyper-parameters suggested in the original implementations listed in \Cref{appendix:implementation} without further tuning. We use the same number of hidden units in the two layers of MLP for interactions, and tune the number of hidden units in \{64, 128, 256, 512\}. For all experiments, we use Adam \citep{kingma2017adam} and Adagrad \citep{john2011adaptive} as optimizers for training models. We tune the initial learning rate for all experiments in the range of $1e^{-7}$ to $0.1$ with a multiplicative step size of 3. Adam has a slight edge over Adagrad in most experiments and the best learning rate for Adam are $3e^{-6}$ and $1e^{-5}$ for most experiments. We use a batch size of 32 for text classification tasks, and a batch size of 64 for image classification tasks. For all configurations of the models and datasets, we train the model for 100,000 steps in text classification tasks, and 50 epochs in image classification tasks. We pick the best checkpoint on the validation set (if provided) for evaluation. 

% MLP layers

\section{Results on different combinations of ranking losses and interaction patterns}
\label{appendix:results}
Here we provide the experiment results on different combinations of loss functions and interaction patterns. Besides PairLogLoss and ApproxNDCG included in Section~\ref{sec:exp}, we also include results from Gumbel-ApproxNDCG and MSE here. Note that we use the rescaled MSE as suggested in \citep{hui2021evaluation} on ImageNet dataset, since there is only 1 correct class in 1,000 classes, and regular MSE performs poorly on such imbalanced data. In particular, as the number of combinations is large, we display the results with respect to each dataset and instance encoder in each table for better visualization and comparison. Then in each table, we group the results from different combinations of loss functions and interaction patterns according to the three evaluation metrics. We underline the top 3 combinations (ties are included) under each metric, and use $*$ to mark the best. The results of all datasets and instance encoders are shown in Tables \ref{tab:goemo-bert} - \ref{tab:cifar10-vgg}.

Firstly, the conclusions on loss functions, interactions and evaluation metrics from the tables are the same as those in Sections~\ref{sec:exp-loss}~and~\ref{sec:exp-structure}. The Gumbel-ApproxNDCG performs well when more complicated interactions such as latent cross and concatenating embeddings are applied. MSE achieves the best top-1 Error in GoEmotions with BERT and CIFAR-10 with VGG16. It indicates that MSE can be effective in optimizing top-1 metrics in certain tasks, which aligns with the observation in \citet{hui2021evaluation}. However, MSE is not suitable for direct application in problems with highly imbalanced correct and incorrect classes as on the ImageNet classification task. A few rescaling techniques with hyper-parameters need to be applied for MSE to perform properly, which also increase the burden of hyper-parameters tuning. 
Besides, adding different combinations of losses and interaction patterns can always improve the performance, as indicated by $*$ under each metric. Note that we did not further tune the hyper-parameters of instance encoders for different architectures of Rank4Class. It is possible that the hyper-parameters of encoders are more suitable for the classical MCC models, and fine-tuning may further boost the performance of Rank4Class. Last but not least, we see that NDCG is more informative in evaluating MCC performance than Top-5 Error which creates many ties. For example, In \Cref{tab:mind-bert} for MIND with BERT, Top-5 Error fails to find the best performing model, while NDCG@5 can successfully differentiate them. 

% Overall performance, a table to show the best performing combinations of losses and interactions. 

\begin{table}[]
\centering
\caption{Results on GoEmotions with BERT as text encoder.}
\label{tab:goemo-bert}
\begin{tabular}{@{}ccccccc@{}}
\toprule
\multicolumn{7}{c}{GoEmotions + BERT} \\ \midrule
Metric & Interaction & SoftmaxCE & PairLogLoss & ApproxNDCG & \begin{tabular}[c]{@{}c@{}}Gumbel\\ ApproxNDCG\end{tabular} & MSE \\ \midrule
\multirow{3}{*}{Top-1 Error} & dot product & 41.00 & 41.79 & 41.42 & 41.11 & 41.20 \\
 & LC+MLP & 41.35 & 41.98 & {\ul 40.65}* & {\ul 40.70} & 41.31 \\
 & Concat+MLP & 41.02 & 41.55 & {\ul 40.92} & 41.59 & 41.63 \\ \midrule
\multirow{3}{*}{Top-5 Error} & dot product & 13.35 & {\ul 12.00} & 13.00 & 13.80 & 15.00 \\
 & LC+MLP & 12.85 & 12.25 & 12.40 & 12.85 & 14.30 \\
 & Concat+MLP & 12.40 & {\ul 11.85}* & {\ul 12.20} & 12.90 & 13.85 \\ \midrule
\multirow{3}{*}{NDCG@5} & dot product & 73.90 & 74.27 & 74.02 & 73.71 & 73.00 \\
 & LC+MLP & 74.11 & 74.22 & {\ul 74.62}* & {\ul 74.53} & 73.50 \\
 & Concat+MLP & 74.36 & 74.19 & {\ul 74.61} & 74.05 & 73.55 \\ \bottomrule
\end{tabular}
\end{table}

\begin{table}[]
\centering
\caption{Results on GoEmotions with ELECTRA as text encoder.}
\label{tab:goemo-electra}
\begin{tabular}{@{}ccccccc@{}}
\toprule
\multicolumn{7}{c}{GoEmotions + ELECTRA} \\ \midrule
Metric & Interaction & SoftmaxCE & PairLogLoss & ApproxNDCG & \begin{tabular}[c]{@{}c@{}}Gumbel\\ ApproxNDCG\end{tabular} & MSE \\ \midrule
\multirow{3}{*}{Top-1 Error} & dot product & 38.45 & 39.78 & {\ul 37.67} & 38.63 & 37.97 \\
 & LC+MLP & 38.69 & 39.13 & {\ul 37.67} & 38.45 & {\ul 37.36}* \\
 & Concat+MLP & 39.00 & 39.48 & 37.80 & 38.06 & 38.15 \\ \midrule
\multirow{3}{*}{Top-5 Error} & dot product & 10.45 & {\ul 8.80} & 9.25 & 10.65 & 10.30 \\
 & LC+MLP & 9.30 & 8.90 & {\ul 8.65}* & 10.00 & 10.15 \\
 & Concat+MLP & 9.75 & {\ul 8.70} & 9.15 & 9.90 & 10.40 \\ \midrule
\multirow{3}{*}{NDCG@5} & dot product & 76.96 & 77.17 & {\ul 78.03} & 76.72 & 77.37 \\
 & LC+MLP & 77.45 & 77.49 & {\ul 78.30}* & 77.41 & 77.69 \\
 & Concat+MLP & 77.06 & 77.47 & {\ul 77.97} & 77.35 & 77.24 \\ \bottomrule
\end{tabular}
\end{table}

\begin{table}[]
\centering
\caption{Results on MIND with BERT as text encoder.}
\label{tab:mind-bert}
\begin{tabular}{@{}ccccccc@{}}
\toprule
\multicolumn{7}{c}{MIND + BERT} \\ \midrule
Metric & Interaction & SoftmaxCE & PairLogLoss & ApproxNDCG & \begin{tabular}[c]{@{}c@{}}Gumbel\\ ApproxNDCG\end{tabular} & MSE \\ \midrule
\multirow{3}{*}{Top-1 Error} & dot product & 30.90 & 31.21 & 30.97 & 30.88 & 30.71 \\
 & LC+MLP & {\ul 30.20}* & {\ul 30.51} & 30.81 & 30.79 & 30.56 \\
 & Concat+MLP & {\ul 30.51} & 31.29 & 30.77 & 30.92 & 30.71 \\ \midrule
\multirow{3}{*}{Top-5 Error} & dot product & 6.15 & {\ul 5.25} & {\ul 5.25} & 5.75 & 6.95 \\
 & LC+MLP & 5.35 & 5.30 & {\ul 5.25} & 5.45 & 7.15 \\
 & Concat+MLP & 5.55 & {\ul 5.25} & {\ul 5.05}* & 5.65 & 7.30 \\ \midrule
\multirow{3}{*}{NDCG@5} & dot product & 82.78 & 83.25 & 83.36 & 83.11 & 82.59 \\
 & LC+MLP & {\ul 83.56}* & {\ul 83.54} & {\ul 83.43} & 83.28 & 82.48 \\
 & Concat+MLP & 83.31 & 83.27 & 83.35 & 83.16 & 82.37 \\ \bottomrule
\end{tabular}
\end{table}

\begin{table}[]
\centering
\caption{Results on MIND with ELECTRA as text encoder.}
\label{tab:mind-electra}
\begin{tabular}{@{}ccccccc@{}}
\toprule
\multicolumn{7}{c}{MIND + ELECTRA} \\ \midrule
Metric & Interaction & SoftmaxCE & PairLogLoss & ApproxNDCG & \begin{tabular}[c]{@{}c@{}}Gumbel\\ ApproxNDCG\end{tabular} & MSE \\ \midrule
\multirow{3}{*}{Top-1 Error} & dot product & 27.09 & 27.72 & 26.65 & 26.77 & 26.99 \\
 & LC+MLP & 26.89 & 27.79 & 26.74 & {\ul 26.55} & 26.92 \\
 & Concat+MLP & 26.78 & 27.53 & {\ul 26.40}* & {\ul 26.55} & 26.65 \\ \midrule
\multirow{3}{*}{Top-5 Error} & dot product & 4.25 & {\ul 3.75} & 4.30 & 4.25 & 5.35 \\
 & LC+MLP & 4.04 & {\ul 3.75} & {\ul 3.70}* & 3.90 & 5.20 \\
 & Concat+MLP & 4.25 & {\ul 3.70}* & 4.30 & 3.95 & 5.55 \\ \midrule
\multirow{3}{*}{NDCG@5} & dot product & 85.73 & 85.74 & 85.89 & 85.89 & 85.26 \\
 & LC+MLP & 85.67 & 85.71 & {\ul 86.13}* & {\ul 86.09} & 85.31 \\
 & Concat+MLP & 85.81 & 85.85 & {\ul 86.05} & 85.99 & 85.25 \\ \bottomrule
\end{tabular}
\end{table}

\begin{table}[]
\centering
\caption{Results on ImageNet with ResNet50 as image encoder.}
\label{tab:imagenet-resnet}
\begin{tabular}{@{}ccccccc@{}}
\toprule
\multicolumn{7}{c}{ImageNet + ResNet50} \\ \midrule
Metric & Interaction & SoftmaxCE & PairLogLoss & ApproxNDCG & \begin{tabular}[c]{@{}c@{}}Gumbel\\ ApproxNDCG\end{tabular} & \begin{tabular}[c]{@{}c@{}}MSE\\ (rescaled)\end{tabular} \\ \midrule
\multirow{3}{*}{Top-1 Error} & dot product & 23.74 & 23.64 & 23.62 & 23.65 & 26.46 \\
 & LC+MLP & 23.62 & 23.70 & 23.65 & {\ul 23.60} & 24.18 \\
 & Concat+MLP & 23.66 & 23.79 & {\ul 23.61} & {\ul 23.58}* & 23.83 \\ \midrule
\multirow{3}{*}{Top-5 Error} & dot product & 6.90 & 6.80 & 6.85 & 6.80 & 8.95 \\
 & LC+MLP & 6.80 & {\ul 6.75}* & 6.80 & 6.80 & 6.90 \\
 & Concat+MLP & {\ul 6.75}* & {\ul 6.75}* & 6.80 & 6.80 & 6.80 \\ \midrule
\multirow{3}{*}{NDCG@5} & dot product & 85.74 & 85.82 & 85.80 & 85.82 & 83.24 \\
 & LC+MLP & {\ul 85.84} & 85.80 & 85.82 & {\ul 85.84} & 85.44 \\
 & Concat+MLP & 85.83 & 85.77 & 85.82 & {\ul 85.85}* & 85.72 \\ \bottomrule
\end{tabular}
\end{table}

\begin{table}[]
\centering
\caption{Results on CIFAR-10 with VGG16 as image encoder.}
\label{tab:cifar10-vgg}
\begin{tabular}{@{}ccccccc@{}}
\toprule
\multicolumn{7}{c}{CIFAR-10 + VGG16} \\ \midrule
Metric & Interaction & SoftmaxCE & PairLogLoss & ApproxNDCG & \begin{tabular}[c]{@{}c@{}}Gumbel\\ ApproxNDCG\end{tabular} & \begin{tabular}[c]{@{}c@{}}MSE\end{tabular} \\ \midrule
\multirow{3}{*}{Top-1 Error} & dot product & 6.56 & 6.43 & 6.41 & 6.45 & 6.48 \\
 & LC+MLP & {\ul 6.40}* & 6.42 & 6.46 & {\ul 6.40}* & 6.42 \\
 & Concat+MLP & 6.43 & 6.44 & 6.44 & 6.47 & {\ul 6.40}* \\ \midrule
\multirow{3}{*}{Top-5 Error} & dot product & 0.20 & {\ul 0.15}* & {\ul 0.15}* & 0.20 & {\ul 0.15}* \\
 & LC+MLP & 0.20 & {\ul 0.15}* & 0.25 & 0.20 & 0.25 \\
 & Concat+MLP & 0.20 & 0.20 & 0.25 & {\ul 0.15}* & 0.35 \\ \midrule
\multirow{3}{*}{NDCG@5} & dot product & 97.19 & 97.21 & {\ul 97.23}* & 97.21 & 97.16 \\
 & LC+MLP & {\ul 97.22} & 97.21 & 97.17 & 97.17 & 97.14 \\
 & Concat+MLP & {\ul 97.23}* & 97.21 & 97.19 & {\ul 97.22} & 97.16 \\ \bottomrule
\end{tabular}
\end{table}

\end{document}